\documentclass[default,iicol]{sn-jnl}

\usepackage{enumitem}
\usepackage{makecell}
\usepackage{amsthm}
\usepackage{booktabs}
\usepackage{array}
\usepackage{natbib}

\jyear{2021}%

\theoremstyle{thmstyleone}%
\newtheorem{theorem}{Theorem}
\newtheorem{proposition}[theorem]{Proposition}%

\theoremstyle{thmstyletwo}%

\theoremstyle{thmstylethree}%
\newtheorem{definition}{Definition}%

\begin{document}

\title[Article Title]{KrwEmd: Revising the Imperfect-Recall Abstraction from Forgetting Everything}








\author[1,2]{\fnm{Yanchang} \sur{Fu}}\email{fuyanchang2020@ia.ac.cn}
\author*[2]{\fnm{Qiyue} \sur{Yin}}\email{qyyin@nlpr.ia.ac.cn}
\author[2]{\fnm{Shengda} \sur{Liu}}\email{shengda.liu@ia.ac.cn}
\author[2]{\fnm{Pei} \sur{Xu}}\email{pei.xu@ia.ac.cn}
\author[2]{\fnm{Kaiqi} \sur{Huang}}\email{kqhuang@nlpr.ia.ac.cn}

\affil[1]{\orgdiv{School of Artificial Intelligence}, \orgname{University of Chinese Academy of Sciences}}

\affil[2]{\orgdiv{Center for Research on Intelligent System and Engineering \& Key Laboratory of Complex System Intelligent Control and Decision}, \orgname{Institute of Automation, Chinese Academy of Sciences}}


\abstract{Excessive abstraction is a critical challenge in hand abstraction—a task specific to games like Texas hold'em—when solving large-scale imperfect-information games, as it impairs AI performance. This issue arises from extreme implementations of imperfect-recall abstraction, which entirely discard historical information. This paper presents KrwEmd, the first practical algorithm designed to address this problem. We first introduce the k-recall winrate feature, which not only qualitatively distinguishes signal observation infosets by leveraging both future and, crucially, historical game information, but also quantitatively captures their similarity. We then develop the KrwEmd algorithm, which clusters signal observation infosets using earth mover's distance to measure discrepancies between their features. Experimental results demonstrate that KrwEmd significantly improves AI gameplay performance compared to existing algorithms.}




\maketitle

\section{Introduction}\label{sec1}

Abstraction refers to the process of simplifying complex games by grouping similar states or actions into broader categories, thereby improving decision-making and computational efficiency. Heads-Up No-Limit Hold'em (HUNL)—a classic imperfect-information game (IIG) and a critical AI testbed—benefits notably from this technology: abstraction enables AI systems to outperform human experts even with limited computational resources \citep{moravvcik2017deepstack, brown2018superhuman, brown2019superhuman}.


Hand abstraction in Texas Hold’em—framed as an unsupervised representation learning task—groups similar hands into equivalence classes, enabling strategy generalization across related game scenarios. Additionally, imperfect-recall abstraction \cite{waugh2009practical} boosts computational efficiency further by relaxing memory consistency constraints. Notably, to streamline computations even more, this abstraction is often implemented in an extreme form: it entirely discards historical information and focuses solely on current and future data \citep{gilpin2006competitive, gilpin2007better, gilpin2007potential, gilpin2008expectation, ganzfried2014potential}.

However, recent research \citep{fu2025beyond} has shown that this extreme implementation of imperfect-recall abstraction can lead to \textbf{excessive abstraction}—a problem where hands with significant differences are grouped into the same category. As illustrated in Figure \ref{fig:considering_historical_vs_future_only}, when historical information is discarded, hands A and B appear identical; yet when historical context is considered, the two differ substantially. This issue exacerbates as the game progresses, with excessive abstraction becoming more severe in late-game 
phases and impairing the solver's global perspective.

Potential-aware outcome isomorphism (PAOI) was introduced to characterize excessive abstraction: it represents the maximum capability of imperfect-recall abstraction (which entirely discards historical information) to identify hand equivalence classes, with discriminative power far weaker than the ground truth. In contrast, full-recall outcome isomorphism (FROI) was proposed to integrate historical information into abstraction—this enables finer distinctions between hands and mitigates excessive abstraction. Notably, FROI exhibits far stronger ability to identify hand equivalence classes than PAOI, and its superiority has been validated by experiments~\cite{fu2025beyond}.

\begin{figure}[h]%
  \centering
  \vspace{-15pt}
  \includegraphics[width=0.36\textwidth]{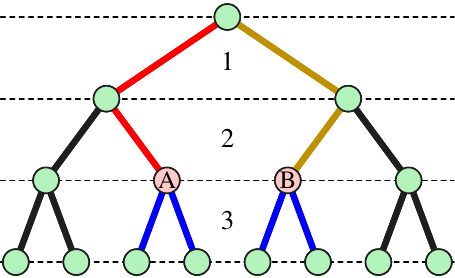} 
  \caption{In a 3-phase hand abstraction task for games, the objective is to group hands A and B—they share the same future trajectory despite following distinct historical paths (consistent information is represented by paths of the same color). However, when an imperfect-recall abstraction that entirely discards historical information is adopted, both hands are assigned identical features.}
  \label{fig:considering_historical_vs_future_only}
\end{figure}

However, two critical gaps remain: (1) FROI lacks the capability for further clustering and cannot adjust the number of categories, making it not a true hand abstraction algorithm. (2) Previous comparisons between FROI and PAOI relied on maximum category counts—while FROI yields more categories, it remains unclear whether integrating historical information still outperforms entirely discarding it when the number of hand equivalence classes is equal.
 
This paper advances hand abstraction by addressing these critical gaps through a novel framework for constructing hand features based on winrates. We introduce the k-recall winrate feature, which retains approximately 90\% of the categorical discriminability of FROI while using significantly less data, enabling quantitative similarity assessment. By combining this feature with earth mover's distance (EMD), we develop KrwEmd—the first practical hand abstraction algorithm integrating historical information. 

Experiments in the Numeral211 Hold’em environment validate our approach: first, when the number of clusters is matched to PAOI, KrwEmd outperforms PAOI significantly. Second, in broader clustering scenarios—when compared against other hand abstraction algorithms (including the state-of-the-art (SOTA) potential-aware abstraction with earth mover's distance, PAAEMD~\cite{ganzfried2014potential}) under the same cluster count conditions—KrwEmd also demonstrates superior performance. These results collectively confirm the value of integrating historical information in hand abstraction.

\subsection{Related Works}

Hand abstraction in Texas Hold'em builds on foundational work by \citet{shi2001abstraction} and \citet{billings2003approximating}, who introduced game simplification via manual hand grouping. \citet{gilpin2006competitive} pioneered automated methods, followed by \citet{gilpin2007lossless}, who formalized \textbf{lossless isomorphism (LI)} for ordered signals in hold'em—though its low compression ratio limited scalability. This shifted focus to lossy methods: \citet{gilpin2007better} proposed expectation-based clustering(\textbf{EHS; expected hand strength}); \citet{gilpin2007potential} introduced potential-aware abstractions; and \citet{ganzfried2014potential} developed the state-of-the-art \textbf{PAAEMD}, which was integrated into top AI systems like DeepStack, Libratus and Pluribus~\citep{moravvcik2017deepstack, brown2018superhuman, brown2019superhuman}. These methods, however, rely on extreme imperfect-recall abstraction, discarding historical information entirely. Later, \citet{fu2025beyond} formally mathematically modeled the hand abstraction task in Texas Hold'em as \textbf{signal observation abstraction}, proposed a \textbf{resolution bound} metric to evaluate such abstraction, revealed the problem of excessive abstraction by introducing PAOI, and further put forward a framework for integrating historical information—namely FROI.

Following the breakthroughs of AI systems like DeepStack, Libratus and Pluribus, efforts have emerged to solve poker AI using deep learning (DL) and deep reinforcement learning (DRL)—approaches that encode hands without explicitly extracting their strength information, achieving progress in solution development though falling short of surpassing Libratus’ performance \citep{brown2019deep, brown2020combining, lidouble, liu2023model, DBLP:conf/iclr/McAleerFLS23, heinrich2016deep, steinberger2020dream, zhao2022alphaholdem, meng2023efficient}. One key reason lies in the inherent challenge within large-scale imperfect-information games (IIGs): DL/DRL methods face overlapping objectives, as they must simultaneously learn hand representations and optimize strategies—both derived indirectly from trajectory outcomes. This dual burden ultimately limits training efficiency in large-scale poker scenarios.

\section{Background and Notation}

Texas Hold'em-style poker games (referred to as hold'em games throughout this paper) are typically modeled as imperfect-information games (IIGs), but hand abstraction—core to simplifying these games—requires a more tailored framework to disentangle the independent dimensions of chance-dealt signals (hands) and rational players' actions. This section first reviews foundational IIG concepts, then introduces the game model Signal Observation Ordered Games (SOOGs) proposed by \citet{fu2025beyond}—a constrained subset of IIGs designed explicitly for modeling hand abstraction tasks—and finally formalizes hand abstraction task within the SOOG framework.

\subsection{Imperfect-Information Games (IIGs)}
We begin with the standard IIG definition, which forms the basis for SOOGs.

\begin{definition}[Imperfect-Information Game (IIG)] \label{def:iig}
An imperfect-information game \(\mathcal{G}\) consists of the following fundamental elements:

\begin{itemize}[left=0cm]
    \item \(\mathcal{N}_c = \mathcal{N} \cup \{c\}\): A finite collection of players. Here, \(\mathcal{N} = \{1, \ldots, N\}\) stands for rational players (the actual decision-makers in the game), while \(c\) denotes a special player representing stochastic elements (often called \textbf{nature} or \textbf{chance}).
    
    \item \(\mathcal{A}\): The complete set of possible actions that can be taken by any player in \(\mathcal{N}_c\) during the game.
    
    \item \(H\): A finite set of histories, each being a sequence of successive player actions. The initial history is the empty sequence \(h^o \in H\). We write \(h \sqsubset h'\) to indicate that \(h\) is a predecessor of \(h'\) (and \(h'\) is a successor of \(h\)); formally, this means there exist actions \(a_1, a_2, \ldots, a_k \in \mathcal{A}\) such that \(h' = h \cdot a_1 \cdot a_2 \cdot \ldots \cdot a_k\) (i.e., \(h'\) is formed by appending the sequence of actions to \(h\)).
    
    \item \(Z \subseteq H\): The subset of terminal histories—histories with no subsequent actions—where the game terminates.
    
    \item \(\rho: H \setminus Z \mapsto \mathcal{N}_c\): An action assignment function that specifies which player is to act at each non-terminal history. This function partitions \(H\backslash Z\) into \(\{H_1, \ldots, H_N, H_c\}\), where \(H_i = \{ h \in H\backslash Z \mid \rho(h) = i \in \mathcal{N}_c \}\) represents the set of histories where player \(i\) is the one to act.
    
    \item \(A: H \setminus Z \mapsto 2^{\mathcal{A}}\): A function defining legal actions, mapping each non-terminal history to the set of actions that can be legally taken at that history. Here, \(2^{\mathcal{A}}\) denotes the power set of \(\mathcal{A}\) (i.e., all possible subsets of \(\mathcal{A}\)).
    
    \item \(\zeta\): A probability distribution function for chance events. For each non-terminal history \(h \in H_c\) (where nature acts), \(  \zeta(h, \cdot): A(h) \to [0,1]\) specifies a probability distribution over the legal actions in \(A(h)\), with the normalization condition \(\sum_{a \in A(h)}  \zeta(h, a) = 1\) ensuring it is a valid distribution.
    
    \item \(u = (u_i)_{i \in \mathcal{N}}\): A tuple of utility functions. For each rational player \(i \in \mathcal{N}\), \(u_i: Z \to \mathbb{R}\) assigns a real payoff to \(i\) at each terminal history \(z \in Z\), which determines the reward \(i\) receives when the game ends at \(z\).
    
    \item \(\mathcal{I} = (\mathcal{I}_i)_{i \in \mathcal{N}}\): A set of information partitions. For each \(i \in \mathcal{N}\), \(\mathcal{I}_i\) divides \(H_i\) into mutually exclusive \textbf{information sets (infosets)}. If two histories \(h, h' \in H_i\) belong to the same infoset \(I \in \mathcal{I}_i\), player \(i\) cannot tell \(h\) and \(h'\) apart; this implies \(A(h) = A(h')\), so we use the overloaded notation \(A(I) := A(h) = A(h')\) for simplicity.
\end{itemize}
\end{definition}

\subsection{Strategies and Solution Concepts for IIGs}

In an IIG, each player \(i \in \mathcal{N}\) employs a \textbf{(behavioral) strategy} \(\sigma_i\), drawn from the set of all possible strategies \(\Sigma_i\). Such a strategy specifies, for each of the player’s infosets, a probability distribution over the available actions. Formally, \(\sigma_i = \{ \sigma_i(I, \cdot) \mid I \in \mathcal{I}_i \}\), where \(\sigma_i(I, \cdot): A(I) \to [0,1]\) satisfies normalization condition (i.e. \(\sum_{a \in A(I)} \sigma_i(I, a) = 1\)) for every \(I \in \mathcal{I}_i\). When all rational players choose their strategies, they form a strategy profile \(\sigma = (\sigma_1, \ldots, \sigma_N)\), with the set of all such profiles denoted \(\Sigma = \times_{i \in \mathcal{N}} \Sigma_i\).

To compute payoffs under \(\sigma \in \Sigma\), we first define the probability of reaching each history. For any \(h \in H\), \(\pi_c(h) = \prod\nolimits_{h' \sqsubset h, h' \cdot a \sqsubseteq h, \rho(h') = c} \zeta(h', a)\) captures nature’s contribution to reaching \(h\), where \(h' \sqsubseteq h\) includes \(h' = h\) (extending \(h' \sqsubset h\)). For rational player \(i\in \mathcal{N}\), \(\pi_i^\sigma(h) = \prod\nolimits_{h' \sqsubset h, h' \cdot a \sqsubseteq h, \rho(h') = i} \sigma_i(I[h'], a)\) reflects their strategy’s contribution, where \(I[h'] \in \mathcal{I}_i\) is the infoset containing \(h'\). Since nature’s choices are independent of player strategies, \(\pi_c(h)\) is invariant to \(\sigma\); thus, when \(\sigma\) is given, we make no distinction between \(\pi_c^\sigma\) and \(\pi_c\) (omitting \(\sigma\) from the notation). The total reaching probability is \(\pi^\sigma(h) = \prod_{i \in \mathcal{N}_c} \pi_i^\sigma(h)\), and player \(i\)’s expected payoff under \(\sigma\) is \(u_i(\sigma) = \sum_{z \in Z} \pi^\sigma(z) \cdot u_i(z)\).

An IIG has \textbf{perfect-recall} if every player \(i \in \mathcal{N}\) retains all past actions and observations—no information is forgotten. Otherwise, it has \textbf{imperfect-recall}. Under perfect-recall, a strategy profile \(\sigma^* = (\sigma_1^*, \ldots, \sigma_N^*)\) is a \textbf{Nash equilibrium} if, for all \(i \in \mathcal{N}\) and \(\sigma_i \in \Sigma_i\), \(u_i(\sigma^*) \geq u_i(\sigma_i, \sigma_{-i}^*)\) (where \(\sigma_{-i}^*\) excludes \(i\)’s strategy). Every finite perfect-recall IIG has at least one such equilibrium.

For two-player zero-sum IIGs (\(N=2\), \(u_1(z) = -u_2(z)\) for all \(z \in Z\)) with perfect-recall, \textbf{exploitability} \(\epsilon(\sigma) = \frac{\epsilon_1(\sigma)+\epsilon_2(\sigma)}{2}\) measures how far \(\sigma = (\sigma_1, \sigma_2)\) is from a Nash equilibrium \(\sigma^*\), where
\[\left\{
\begin{aligned}
\epsilon_1(\sigma) &= u_1(\sigma^*) - \min_{\sigma'_2 \in \Sigma_2} u_1(\sigma_1, \sigma'_2), \\
\epsilon_2(\sigma) &= u_2(\sigma^*) - \min_{\sigma'_1 \in \Sigma_1} u_2(\sigma'_1, \sigma_2)
\end{aligned}
\right.\]
and \(\epsilon(\sigma) \leq \varepsilon\) defines an \(\varepsilon\)-Nash equilibrium.

\subsection{Signal Observation Ordered Games (SOOGs)}
Traditional IIGs merge chance event (hand dealing) and rational players’ action sequences into monolithic infosets, precluding independent analysis of these two core dimensions—critical for hand abstraction. SOOGs resolve this by imposing constraints on IIGs to explicitly model the structure of hold'em games.

First, history operators are defined to separate chance and non-chance components:

\begin{itemize}[left=0cm]
    \item \textbf{Trace extraction operator} \(\mathcal{H}_{-i}(\cdot)\): Retains actions of all players except \(i\) in a history, replacing \(i\)'s actions with wildcard \(\emptyset_i\). For example, if \(h = a_i^1 \cdot a_j^1 \cdot a_k^1 \cdot a_i^2\) \footnote{In practice, the full history is \(h = h^o \cdot a_i^1 \cdot a_j^1 \cdot a_k^1 \cdot a_i^2\), where the initial history \(h^o\) is omitted; here, \(a_i^1\) and \(a_i^2\) denote player i's actions}, then \(\mathcal{H}_{-i}(h) = \emptyset_i \cdot a_j^1 \cdot a_k^1 \cdot \emptyset_i\).
    
    \item \textbf{Sequence extraction operator} \(\tilde{\mathcal{H}}_{-i}(\cdot)\): Defined as \(\tilde{\mathcal{H}}_{-i}(\cdot) := \mathcal{E} \circ \mathcal{H}_{-i}\), where \(\mathcal{E}\) (wildcard elimination operator) removes all wildcards. For the above \(h\), \[\tilde{\mathcal{H}}_{-i}(h) = \mathcal{E}(\emptyset_i \cdot a_j^1 \cdot a_k^1 \cdot \emptyset_i)= a_j^1 \cdot a_k^1.\]
    
    \item \(\mathcal{H}_{i}(\cdot)\) (trace extraction for \(i\)): Retains \(i\)'s actions with others replaced by wildcards, defined as \(\mathcal{H}_{i}(\cdot) := \mathcal{E} \circ \bigcirc_{\substack{j \in \mathcal{N}_c \\ j \neq i}} \mathcal{H}_{-j}\) ( \(\bigcirc\) denotes sequential composition: \(\bigcirc_{k=1}^N f_k := f_N \circ \dots \circ f_1\) ). For the above \(h\), \(\mathcal{H}_{i}(h) = a_i^1 \cdot \emptyset_j \cdot \emptyset_k \cdot a_i^2\).
    
    \item \(\tilde{\mathcal{H}}_{i}(\cdot)\) (sequence extraction for \(i\)): Eliminates wildcards from \(\mathcal{H}_{i}(h)\) via \(\mathcal{E}\), i.e., \(\tilde{\mathcal{H}}_{i}(\cdot) := \mathcal{E} \circ \mathcal{H}_{i}\). For the above \(h\), \(\tilde{\mathcal{H}}_{i}(h) = a_i^1 \cdot a_i^2\).
\end{itemize}

A critical property of the trace extraction operator is \textbf{trace complementarity}: Two traces are complementary if, at every position, one contains a concrete action of player $i\in \mathcal{N}_c$ and the other contains \(\emptyset_i\) (no overlapping actions or simultaneous wildcards). For any history \(h\in H\), \(\mathcal{H}_i(h)\) and \(\mathcal{H}_{-i}(h)\) satisfy this property. Complementary traces can be spliced into a formal history via \(h = \mathcal{H}_i(h) \oplus \mathcal{H}_{-i}(h)\), where \(\oplus\) resolves positions by prioritizing concrete actions over wildcards (the spliced history is not guaranteed to be legal).

\begin{definition}[Signal Observation Ordered Game (SOOG)] \label{def:soog}
A signal observation ordered game is an IIG \(\mathcal{G}\) (satisfying Definition~\ref{def:iig}) with the following additional constraints:
\begin{itemize}[left=0cm]
    \item \(\gamma: H \mapsto \mathbb{N}^+\): A function that partitions histories into phases. For any history \(h \in H\), \(\gamma(h)\) corresponds to the count of chance histories (i.e., histories where nature acts) along the path from the initial history \(h^o\) to \(h\) (including \(h\)), thus determining the phase of \(h\). Let \(\Gamma = \max_{h \in H} \gamma(h)\) represent the final phase; it is noteworthy that \(\gamma(h^o) = 1\), meaning \(h^o\) is a chance history.
    
    \item Chance actions reveal signals: Let \(\Theta\) denote the total set of signals, defined as \(\Theta = \{\tilde{\mathcal{H}}_c(h) \mid h \in H\}\). The behavior of chance actions depends solely on the signals already revealed: for any two chance histories \(h, h' \in H_c\), if \(\tilde{\mathcal{H}}_c(h) = \tilde{\mathcal{H}}_c(h')\), then their available actions are identical (\(A(h) = A(h')\)) and their action probability distributions coincide (\(\zeta(h, a) = \zeta(h', a)\) for all \(a \in A(h)\)). Given that \(\zeta\) at chance histories is only related to previously dealt signals, and all legal actions \(a\) at such histories serve to reveal new signals, we can replace \(\zeta(h, a)\) with a more targeted function \(\xi(\theta, \theta')\)—where \(\theta \in \Theta\) is the signal revealed at the current chance history, and \(\theta' \in \Theta\) is the new signal uncovered by action \(a\).
    
    
    \item Signal-action separability: For any three histories \(h_1, h_2, h_1' \in H\), if two conditions are satisfied: (1) \(h_1\) and \(h_1'\) have identical signals (i.e., identical chance action sequences), corresponding to \(\tilde{\mathcal{H}}_c(h_1) = \tilde{\mathcal{H}}_c(h_1')\); and (2) \(h_1\) and \(h_2\) share the same non-chance action traces, as denoted by \(\mathcal{H}_{-c}(h_1) = \mathcal{H}_{-c}(h_2)\)—then there must exist a history \(h_2' \in H\) such that \(\tilde{\mathcal{H}}_c(h_2') = \tilde{\mathcal{H}}_c(h_2)\) and \(\mathcal{H}_{-c}(h_2') = \mathcal{H}_{-c}(h_1')\).
    
    \item Signal observation partitions: Let \(\Psi = (\Psi_i)_{i \in \mathcal{N}}\), where each \(\Psi_i\) (referred to as signal observation infosets for player \(i\)) constitutes a partition of \(\Theta\). Let \(\vartheta = (\vartheta_i)_{i \in \mathcal{N}}\) be a tuple of observation functions, with each \(\vartheta_i: \Theta \mapsto \Psi_i\) mapping a signal \(\theta \in \Theta\) to its associated signal observation infoset \(\psi \in \Psi_i\). Signals falling within the same \(\psi \in \Psi_i\) cannot be distinguished by player \(i\).
    
    \item Phase-specific subsets: For each phase \(r \in \{1, \dots, \Gamma\}\), we define the following subsets:
\(H^{(r)} = \{h \in H \mid \gamma(h) = r\}\) (histories in phase \(r\)), 
\(H_i^{(r)} = H^{(r)} \cap H_i\) (histories in phase \(r\) where player \(i\) acts), 
\(Z^{(r)} = Z \cap H^{(r)}\) (terminal histories in phase \(r\)), 
\(\Theta^{(r)} = \{\tilde{\mathcal{H}}_c(h) \mid h \in H^{(r)}\}\) (signals in phase \(r\)), 
\(\Psi_i^{(r)} = \{ \psi \cap \Theta^{(r)} \mid \psi \in \Psi_i \}\) (signal observation infosets of player \(i\) in phase \(r\)), 
\(\Psi^{(r)} = (\Psi_i^{(r)})_{i \in \mathcal{N}}\) (signal observation partitions in phase \(r\)), 
\(\mathcal{I}_i^{(r)} = \{I \in \mathcal{I}_i \mid I \subseteq H^{(r)}\}\) (infosets of player \(i\) in phase \(r\)), 
and \(\mathcal{I}^{(r)} = (\mathcal{I}_i^{(r)})_{i \in \mathcal{N}}\) (information partitions in phase \(r\)).
    
    \item Criterion for history indistinguishability: For any two histories \(h, h' \in H_i\) (where \(i \in \mathcal{N}\)), \(h\) and \(h'\) belong to the same infoset \(I \in \mathcal{I}_i\) exactly when two conditions are met: their non-chance action traces are identical (\(\mathcal{H}_{-c}(h) = \mathcal{H}_{-c}(h')\)) and the signals they contain are indistinguishable to \(i\) (\(\vartheta_i(\tilde{\mathcal{H}}_c(h)) = \vartheta_i(\tilde{\mathcal{H}}_c(h'))\)).
    
    \item Survival functions: \(\omega = (\omega_i)_{i \in \mathcal{N}}\) is a tuple of survival functions. For each player \(i \in \mathcal{N}\) and history \(h \in H\), 
    \[
    \omega_i(h) = \mathbb{I}_{\{\text{player } i \text{ is still participating at } h\}},
    \]
    where \(\mathbb{I}_{\{\cdot\}}\) is the indicator function (taking the value 1 if the condition is satisfied, and 0 otherwise).
    
    \item Terminal order and utility consistency: In the final phase \(\Gamma\), each signal \(\theta \in \Theta^{(\Gamma)}\) (referred to as final signals) induces a total order \(\preceq_\theta\) over the set \(\mathcal{N}\). This order satisfies reflexivity (for all \(i\), \(i \preceq_\theta i\)), totality (for any \(i, j \in \mathcal{N}\), either \(i \preceq_\theta j\) or \(j \preceq_\theta i\)), and transitivity (if \(i \preceq_\theta j\) and \(j \preceq_\theta k\), then \(i \preceq_\theta k\)). For any terminal history \(z \in Z^{(\Gamma)}\) with \(\theta = \tilde{\mathcal{H}}_c(z)\), if both \(i\) and \(j\) are still participating at \(z\) (i.e., \(\omega_i(z)\omega_j(z) = 1\)) and \(i \preceq_\theta j\), then \(u_i(z) \leq u_j(z)\).
\end{itemize}
\end{definition}

\subsection{Signal Observation Abstraction and Resolution Bound}

The SOOG framework is leveraged to model the hand abstraction task in hold'em games as a signal observation abstraction model, and the resolution bound is introduced as a tool for evaluating the quality of abstraction algorithms without reliance on strategy solving.

Formally, a \textbf{signal observation abstraction profile} is denoted \(\alpha = (\alpha_1, \dots, \alpha_N)\): for each player \(i \in \mathcal{N}\), \(\Psi^{\alpha}_i\) represents the set of abstracted signal observation infosets, and \(\alpha_i: \Theta \mapsto \Psi^{\alpha}_i\) maps each original signal \(\theta \in \Theta\) to an abstracted signal observation infoset \(\psi^\alpha \in \Psi^{\alpha}_i\), where every \(\psi^\alpha\) can be split into signal observation infosets from the original \(\Psi_i\). A \(\textbf{(signal observation) abstracted game}\) \(\mathcal{G}^\alpha\) is produced by substituting the original observation function \(\vartheta_i\) with \(\alpha_i\) in a SOOG \(\mathcal{G}\); this substitution modifies the information structure of the game (not the game itself). The abstracted infoset partition \(\mathcal{I}_i^\alpha\) for player \(i\) follows this rule: two original infosets \(I, I' \in \mathcal{I}_i\) belong to the same abstracted infoset \(I^\alpha \in \mathcal{I}_i^\alpha\) if and only if histories \(h \in I\) and \(h' \in I'\) exist such that \(\alpha_i(\tilde{\mathcal{H}}_c(h)) = \alpha_i(\tilde{\mathcal{H}}_c(h'))\) (indistinguishable abstracted observation of signals) and \(\mathcal{H}_{-c}(h) = \mathcal{H}_{-c}(h')\) (identical non-chance action traces). Notably, \(\mathcal{I}_i^\alpha\) forms a partition of \(\mathcal{I}_i\), meaning every original infoset \(I \in \mathcal{I}_i\) is a subset of exactly one abstracted infoset \(I^\alpha\).

In an abstracted \(\mathcal{G}^\alpha\), each rational player \(i\in \mathcal{N}\) adopts an \(\textbf{abstracted strategy}\) \(\sigma^\alpha_i \in \Sigma^\alpha_i\), where \(\sigma^\alpha_i = \{ \sigma^\alpha_i(I^\alpha, \cdot) \mid I^\alpha \in \mathcal{I}_i^\alpha \}\) and \(\sigma^\alpha_i(I^\alpha, \cdot): A(I^\alpha) \to [0,1]\) is a probability distribution over legal actions—with \(A(I^\alpha) = A(I)\) for all original infosets \(I \subset I^\alpha\), as legal actions are consistent across merged infosets. When all rational players select their abstracted strategies, an abstracted strategy profile \(\sigma^\alpha = (\sigma^\alpha_1, \ldots, \sigma^\alpha_N)\) is formed, and the set of all such profiles is denoted \(\Sigma^\alpha = \times_{i \in \mathcal{N}} \Sigma^\alpha_i\). A key property of these abstracted strategies is that any \(\sigma^\alpha_i \in \Sigma^\alpha_i\) can induce a corresponding original strategy \(\sigma_i \in \Sigma_i\) in \(\mathcal{G}\): for every original infoset \(I \subset I^\alpha\), the action probability distribution of \(\sigma_i\) at \(I\) matches that of \(\sigma^\alpha_i\) at \(I^\alpha\), i.e., \(\sigma_i(I, a) = \sigma^\alpha_i(I^\alpha, a)\) for all \(a \in A(I)\). An abstraction \(\alpha\) is said to have \(\textbf{perfect-recall}\) if the induced abstracted game \(\mathcal{G}^\alpha\) has perfect-recall; otherwise, it is referred to as an \(\textbf{imperfect-recall}\) abstraction.

To evaluate abstraction quality without solving strategies, the concept of \(\textbf{refinement}\) between abstractions is first relied on. For a rational player \(i\), given two signal observation abstractions \(\alpha_i\) and \(\beta_i\), \(\alpha_i\) is said to refine \(\beta_i\) (denoted \(\alpha_i \sqsupseteq \beta_i\)) if every abstracted signal observation infoset \(\psi^\beta \in \Psi_i^{\beta}\) can be partitioned into one or more abstracted signal observation infosets from \(\Psi_i^{\alpha}\); if \(\alpha_i\) refines multiple abstractions \(\alpha_i^1, \dots, \alpha_i^m\) for \(i\), it is called a \(\textbf{common refinement}\) of these abstractions. As proven and intuited by \citet{waugh2009abstraction, fu2025beyond}, more refined (i.e., more granular) abstractions tend to facilitate deriving better-performing solutions when solving the abstracted game.

This refinement concept is extended to evaluating automated abstraction algorithms via the \(\textbf{resolution bound}\)—a measure quantifying the maximum ability of an algorithm to distinguish signal observation infosets. For a given SOOG \(\mathcal{G}\) and rational player \(i\), if a signal observation abstraction \(\alpha_i\) exists such that every abstraction \(\alpha_i'\) generated by the algorithm (across all parameter sets) satisfies \(\alpha_i \sqsupseteq \alpha_i'\) (i.e., \(\alpha_i\) is a common refinement of all the algorithm’s generated abstractions), then \(\alpha_i\) is the algorithm’s resolution bound (strictly defined within the context of \(\mathcal{G}\) and \(i\)). Building on the earlier intuition about refinement, high-quality solutions from the abstracted game induced by an algorithm’s resolution bound tend to outperform those from the abstracted games of the algorithm’s generated abstractions—though this advantage is not strictly guaranteed~\cite{waugh2009abstraction}. Importantly, a finer resolution bound does not inherently ensure better algorithm performance; however, a coarse resolution bound indicates inherent flaws in the algorithm’s ability to distinguish signal observation infosets.

\section{Methods}

The resolution bound functions not only as an evaluation metric for signal observation abstraction algorithms but also as a guideline for their development. In this section, we first derive two signal observation abstraction variants based on winrate calculations: potential-aware winrate isomorphism (PAWI) and k-recall winrate isomorphism (k-RWI). PAWI is a type of signal observation abstraction that discards historical information entirely, and we demonstrate that it acts as the resolution bound for the previously proposed EHS algorithm~\cite{gilpin2007better}. In contrast, k-RWI is a signal observation abstraction that integrates historical information; we further develop the signal observation abstraction algorithm KrwEmd, with k-RWI designated as its resolution bound.

\subsection{Potential-Aware Winrate Isomorphism}

We begin by constructing a flawed signal observation abstraction—one with more pronounced limitations than PAOI \cite{fu2025beyond}. This flawed abstraction groups signal observation infosets into the same category if they produce identical winrates when rollout simulations are performed from the current state.

To support subsequent discussions, we first define the following notations:
\begin{equation}
S_\psi^{(r')} := \left\{ \theta \in \Theta^{(r')} \mid \exists \theta' \in \psi,\ \theta' \sqsubseteq \theta \right\}, \label{eq:extended-signals}
\end{equation}
where \(\psi \subseteq \Psi_i^{(r)}\) (a phase-\(r\) signal observation infoset for a some rational player \(i\)),  \(\theta' \sqsubseteq \theta\) signifies that \(\theta'\) is a prefix of \(\theta\) (or equal), and \(S_\psi^{(r')}\)denotes the set of all phase-\(r'\) signals, \(r' \geq r\),  that extend any signal included in \(\psi\). Additionally, we extend the definition of nature’s reaching probability \(\pi_c(h)\)—originally formulated for histories \(h\)—to be applicable to signals \(\theta\): specifically, let \(\pi_c(\theta) = \prod\nolimits_{\theta' \sqsubseteq \theta} \xi(\theta'', \theta')\), where \(\theta''\) denotes the immediate predecessor of \(\theta'\).

On this basis, we introduce the \textbf{potential-aware winrate feature (PAWF)} for signal observation infosets. Formally, for any \(\psi \in \Psi_i^{(r)}\) in phase \(r\), the PAWF of player \(i\) is defined as:  
\[
pw_i^{(r)}(\psi) = \big( pw_i^{(r),0}(\psi), pw_i^{(r),1}(\psi), \ldots, pw_i^{(r),N}(\psi) \big), \label{eq:pawf}
\]  
where \(\mathcal{N}_{-i} = \mathcal{N} \setminus \{i\}\), and each component of \(pw_i^{(r)}(\psi)\) is defined as follows:  
\begin{itemize}[left=0cm]
    \item \( pw_i^{(r),0}(\psi) \): Represents the probability that player \(i\) ranks lower than at least one other player in the final phase, i.e.,  
      \begin{align}
      pw_i^{(r),0}&(\psi) = \\
      &\frac{\sum\limits_{\theta \in S_\psi^{(\Gamma)}} \pi_c(\theta)\cdot\mathbb{I}_{\left\{ \exists j \in \mathcal{N}_{-i} \text{ s.t. } j \not\preceq_\theta i \text{ and } i \preceq_\theta j \right\}}}{\sum\limits_{\theta \in S_\psi^{(\Gamma)}} \pi_c(\theta)}.
      \end{align}

    \item For \( l \ge 0 \), \( pw_i^{(r),l+1}(\psi) \): Represents the probability that player \(i\) ranks no lower than any other player and ranks higher than exactly \( l \) other players in the final phase, i.e.,  
     \begin{align}
      p&w_i^{(r),l+1}(\psi) = \label{eq:woirl} \\
      &\frac{\sum\limits_{\theta \in S_\psi^{(\Gamma)}} \pi_c(\theta)\cdot\mathbb{I}_{\left\{ \forall j \in \mathcal{N}_{-i}, j \preceq_\theta i \text{ and } \vert \{ j \in \mathcal{N}_{-i} \mid i \preceq_\theta j \} \vert = l \right\}}}{\sum\limits_{\theta \in S_\psi^{(\Gamma)}} \pi_c(\theta)}.\nonumber
      \end{align}
\end{itemize}  

Here, \(\vert \{ \cdot \} \vert\) denotes the cardinality of a set (i.e., the number of elements within the set). 

Notably, in two-player scenarios, only \(pw_i^{(r),0}(\psi)\), \(pw_i^{(r),1}(\psi)\), and \(pw_i^{(r),2}(\psi)\) remain. In hold'em games, these three components correspond respectively to the losing rate, tying rate, and winning rate obtained when enumerating all possible opponent hands against a player’s own hand. It is straightforward to verify that \(\sum_{l=0}^N pw_i^{(r), l}(\psi) = 1\) for each signal observation infoset $\psi$.

Signal observation infosets in the same phase with identical PAWFs form equivalence classes collectively termed \textbf{potential-aware winrate isomorphisms (PAWIs)} — essentially, PAWI partitions such signal observation infosets by matching PAWFs. Since PAWF construction only considers current and future information, PAWI is an extreme imperfect-recall abstraction that entirely discards historical data. Unsurprisingly, like PAOI, it induces excessive abstraction: In late game phases, the number of identifiable equivalence classes of signal observation infosets decreases, contrasting with the lossless isomorphism (LI)~\cite{gilpin2007lossless} abstraction results in HUNL (Table \ref{tab:poi-pawi}), where their equivalence class count increases with each phase. PAWI is even coarser than PAOI, relying solely on weighted winning/tying/losing rates rather than full distributions — a distinction reflected in their data requirements (Table \ref{tab:poi-pawi}): in HUNL, PAWI needs only 3 pieces per phase, while PAOI must enumerate card distributions in non-final phases to capture more complex distributions. This greater coarseness leads to PAWI identifying significantly fewer abstracted signal observation infosets than PAOI; for instance, in the Turn phase, PAWI’s count is only 79.16\% of PAOI’s.

\begin{table}[ht]
  \centering
\caption{The number of signal observation infoset equivalence classes identified by LI, PAWI, and PAOI in each phase of HUNL, where W/O denotes the ratio of such classes identified by PAWI to those by PAOI. Note that the data volume required to identify an equivalence class differs between the two methods: PAWI relies on winning/tying/losing rates, while PAOI needs to enumerate the card distribution of the next phase in non-final phases.}
\resizebox{\columnwidth}{!}{%
\begin{tabular}{cccccc}
\specialrule{1.2pt}{0pt}{0pt}
& & Preflop & Flop & Turn & River \\ \midrule
\multicolumn{1}{c}{\multirow{3}{*}{Counts}} 
& LI & 169 & 1286792 & 55190538 & 2428287420 \\
& PAWI & 169 & 1028325 & 1850624 & 20687 \\
& PAOI & 169 & 1137132 & 2337912 & 20687 \\ 
& W/O (\%) & 100.0 & 90.43 & 79.16 & 100.0 \\ \midrule
\multicolumn{1}{c}{\multirow{2}{*}{\makecell{Data\\Volume}}} 
& PAWI & 3 & 3 & 3 & 3 \\
& PAOI & $\binom{50}{3}$ & $\binom{47}{1}$ & $\binom{46}{1}$ & 3 \\ \bottomrule[1.2pt]
\end{tabular}%
}
\label{tab:poi-pawi}
\end{table}

\begin{table*}[ht]
\centering
\caption{The number of signal observation infoset equivalence classes identified by k-RWI, and k-ROI in each phase and $k$ of HUNL, with W/O indicating the ratio of signal observation infoset equivalence classes identified by k-RWI to those identified by k-ROI.}
\resizebox{\textwidth}{!}{%
\begin{tabular}{ccccccccccc}
\Xhline{1.2pt} \hline
         & Preflop & \multicolumn{2}{c}{Flop} & \multicolumn{3}{c}{Turn}     & \multicolumn{4}{c}{River}                \\
         \cmidrule(lr){2-2} \cmidrule(lr){3-4} \cmidrule(lr){5-7} \cmidrule(lr){8-11}
Recall (k)   & 0       & 0           & 1           & 0       & 1        & 2        & 0     & 1        & 2         & 3         \\ 
k-RWI     & 169     & 1028325     & 1123442     & 1850624 & 34845952 & 37659309 & 20687 & 33117469 & 529890863 & 577366243 \\
k-ROI     & 169     & 1137132     & 1241210     & 2337912 & 38938975 & 42040233 & 20687 & 39792212 & 586622784 & 638585633 \\
W/O (\%) & 100.0   & 90.43       & 90.51       & 79.16   & 89.49    & 89.58    & 100.0 & 83.23    & 90.33     & 90.41     \\ \hline \bottomrule[1.2pt]
\end{tabular}%
}
\label{tbl:krwi-kroi-hunl}
\end{table*}

Though PAWI is coarser than PAOI, we still arrive at the following conclusion:
\begin{proposition} \label{prop:less-ehs}
For a two-player SOOG $\mathcal{G}$, the PAOI abstraction serves as a resolution bound of algorithm EHS.
\end{proposition}

\begin{proof}:

   \textbf{Logic of the EHS algorithm:} For a signal observation infoset $\psi \in \Psi_i^{(r)}$ of player $i$ at phase $r$, calculate its equity as $\text{equity}(\psi) = pw_i^{(r),2}(\psi) - pw_i^{(r),0}(\psi)$. Then perform k-means clustering on all signal observation infosets in $\Psi_i^{(r)}$ to determine the abstracted signal observation infosets, where the equity of the cluster centroid is the average equity of all signal observation infosets in the cluster, and the deviation between the centroid and a signal observation infoset is the difference between the centroid’s equity and the signal observation infoset’s equity. 
   
   \textbf{Key inference:} Any two signal observation infosets $\psi, \psi'$ belonging to the same PAOI equivalence class have identical $pw_i^{(r)}(\psi)$ (and thus identical equity). Therefore, they will be assigned to the same centroid and belong to the same abstracted signal observation infoset after k-means clustering.
\end{proof}

Despite Proposition~\ref{prop:less-ehs}, a similar conclusion does not hold for the PAAEMD algorithm (the current SOTA), as PAWI fails to retain distributional information—this information is, however, a critical factor that PAAEMD accounts for when partitioning abstracted signal observation infosets.

\subsection{K-Recall Winrate Isomorphism}

Inspired by k-ROI\cite{fu2025beyond} (the general form of FROI), we can integrate historical information into PAWI to construct winrate-based equivalence classes, namely k-recall winrate isomorphism (k-RWI)—where ``k-recall" denotes recalling information from the previous k phases. 

For a SOOG \(\mathcal{G}\) with perfect-recall, consider rational player $i$ at phase $r$ and its signal observation infoset \(\psi \in \Psi_i^{(r)}\): All signals in \(\psi\) have predecessors at phases \(r' < r\), and these predecessors belong to a unique signal observation infoset denoted \(\psi^{(r')}\).

We define the \textbf{k-recall winrate feature (k-RWF)} of \(\psi\) (with \(k < r\)) as a \((k+1)(N+1)\)-dimensional numerical array, formally given by:
\begin{align}
r&w_i^{(r,k)}(\psi) = \\
&\big( pw^{(r)}(\psi^{(r)}), pw^{(r-1)}(\psi^{(r-1)}), \ldots, pw^{(r-k)}(\psi^{(r-k)}) \big), \nonumber
\end{align}
where \(\psi^{(r-t)}\) (for \(t = 0, 1, \ldots, k\)) is the unique predecessor signal observation infoset of \(\psi\) at phase \(r-t\), and \(pw^{(r-t)}(\psi^{(r-t)})\) is the PAWF of \(\psi^{(r-t)}\) at phase \(r-t\). Similarly, k-RWI groups signal observation infosets into the same equivalence class if they share identical k-RWFs.

Just as PAWF is a simplification of PAOF, k-RWF is also a simplification of k-ROF. Moreover, it is evident that PAWI is a degenerate case of k-RWI (specifically, k-RWI when \(k=0\), i.e., 0-RWI). Table~\ref{tbl:krwi-kroi-hunl} presents the number of signal observation infoset equivalence classes identified by k-RWI and k-ROI, along with their ratio in HUNL. Notably, while the resolution ratio of PAWI to PAOI can drop below 80\%, when $k$ is set to its maximum value (\(k=r-1\) for phase $r$), the ratio of k-RWI to k-ROI reaches a minimum of nearly 90\%—retaining most of the information. Additionally, k-RWI (\(k>0\)) identifies far more signal observation infoset equivalence classes than PAOI—which serves as the resolution bound for mainstream signal observation abstraction algorithms~\cite{fu2025beyond} (e.g., EHS and the state-of-the-art PAAEMD) that entirely discard historical information.

\subsection{K-Recall Winrate Abstraction with Earth Mover's Distance} \label{sec:KrwEmd}


The advantage of introducing k-RWI lies in that the vector elements of k-RWF are comparable quantities with distinguishable superiority/inferiority—unlike k-ROF, whose elements are far harder to compare in terms of quality. This enables us to further design signal observation abstraction algorithms based on k-RWF through clustering.

For the signal observation infosets $\psi$ and $\psi'$ of player $i$ at phase $r$, we can define the distance of their k-RWF as
\begin{align}
D&(rw_i^{(r,k)}(\psi), rw_i^{(r,k)}(\psi')) \label{equa:distance-of-krwemd} \\
= &\sum_{j=0}^k w_j \cdot \text{Emd}(pw_i^{(r-j)}(\psi^{(r-j)}), pw_i^{(r-j)}(\psi'^{(r-j)})). \nonumber 
\end{align}

Among Equation \eqref{equa:distance-of-krwemd}, \(\text{Emd}(\cdot, \cdot)\) is the operator used to calculate the earth mover's distance. The earth mover's distance can be formulated as a linear programming problem. Given two distributions $\mathbf{p} = (p_1, p_2, \dots, p_n)$ and $\mathbf{q} = (q_1, q_2, \dots, q_m)$ over two sets of points, and a distane matrix $\mathbf{D} = [d_{ij}]_{n\times m}$ representing the ground distances between each point in $\mathbf{p}$ and $\mathbf{q}$, the goal is to find the optimal flow $\mathbf{F} = [f_{ij}]_{n\times m}$ that minimizes the total transportation cost $$\text{Emd}(\mathbf{p}, \mathbf{q}) = \min \sum_{i=1}^{n} \sum_{j=1}^{m} f_{ij} d_{ij}$$ subject to the following constraints:
\begin{align}
    &\displaystyle\sum_{j=1}^{m} f_{ij} = p_i, \quad \forall i = 1, 2, \dots, n \nonumber \\ 
    &\displaystyle\sum_{i=1}^{n} f_{ij} = q_j, \quad \forall j = 1, 2, \dots, m \nonumber \\ 
    &f_{ij} \geq 0, \quad \forall i, j \nonumber 
\end{align}
where $f_{ij}$ represents the amount of flow from $p_i$ to $q_j$. Since it requires solving linear programming equations, the computational complexity of the EMD is sensitive to the dimensionality of the histograms, and approximate algorithms are usually used for larger-scale problems. However, the dimensionality of winrate-based features is small, with a dimension of 3 in a two-player scenario, so we attempt to use a fast algorithm for accurately computing the EMD \citep{BPPH11}. $w_0, \dots, w_k$ are hyperparameters used to control the importance of EMD at each phase $r, \dots, r-k$, and the idea behind this design is to transform the similarity between two signal observation infoset into a linear combination of the EMD distances between their k-recall winrate features' winrates across different phases. We use the k-means++ algorithm \citep{arthur2007k} to further cluster the abstracted signal observation infoset equivalence classes of k-RWI. We named this algorithm KrwEmd.

\section{Experiment}

In this section, we present the experimental setup and results conducted in the Numeral211 Hold'em environment, which is designed specifically to study the hand abstraction task\cite{fu2025beyond}.

\subsection{Numeral211 Hold'em}

Experimental evaluation of hand abstraction in standard Texas Hold’em variants (e.g., HUNL) faces two key challenges: generating multiple hand abstractions via clustering is extremely time-consuming due to their complex state spaces, and evaluating abstraction effectiveness (e.g., calculating exploitability) remains computationally prohibitive with current resources. Meanwhile, existing simplified environments (Leduc Hold'em~\cite{southey2005bayes}, Kuhn Poker~\cite{kuhn1950simplified}) have overly simplistic hand structures—their limited hand categories fail to reflect real poker complexity, making meaningful comparisons of hand abstraction algorithms impossible. As an ideal trade-off, Numeral211 Hold’em is a custom variant tailored for hand abstraction evaluation: it enables rigorous strategy quality assessment on standard consumer computers while retaining diverse hand categories, balancing computational feasibility and evaluative validity.

Numeral211 Hold’em is a two-player game with the following rules: Each player antes 5 chips into the pot before the game starts. It uses a reduced 40-card deck (four suits: spades, hearts, clubs, diamonds; each suit contains 10 numeric cards: 2–9 and Ace). The game proceeds through three betting phases: the first phase deals two private cards to each player, and the next two phases reveal one community card each. Player 1 acts first in Phase 1, while Player 2 acts first in Phases 2 and 3. During each phase, the first actor may either bet a fixed number of chips (10 in Phase 1, 20 in Phases 2–3) to initiate the betting phase or check (pass the action right to the opponent). Once betting starts, players take turns raising (10 in Phase 1, 20 in Phases 2–3; maximum 3 combined raises per phase), calling, or folding. The game advances to the next phase (or concludes) if both players check or one calls the other’s action; folding results in forfeiting the pot to the opponent. Only if the game reaches the final phase without a fold will a showdown occur: players reveal their private cards, and the pot is awarded based on the strength of their hands (specifically, the strongest three-card combination), with hand ranks summarized in Figure~\ref{fig:handranks}.

\begin{figure}[h]%
\centering
\includegraphics[width=0.5\textwidth]{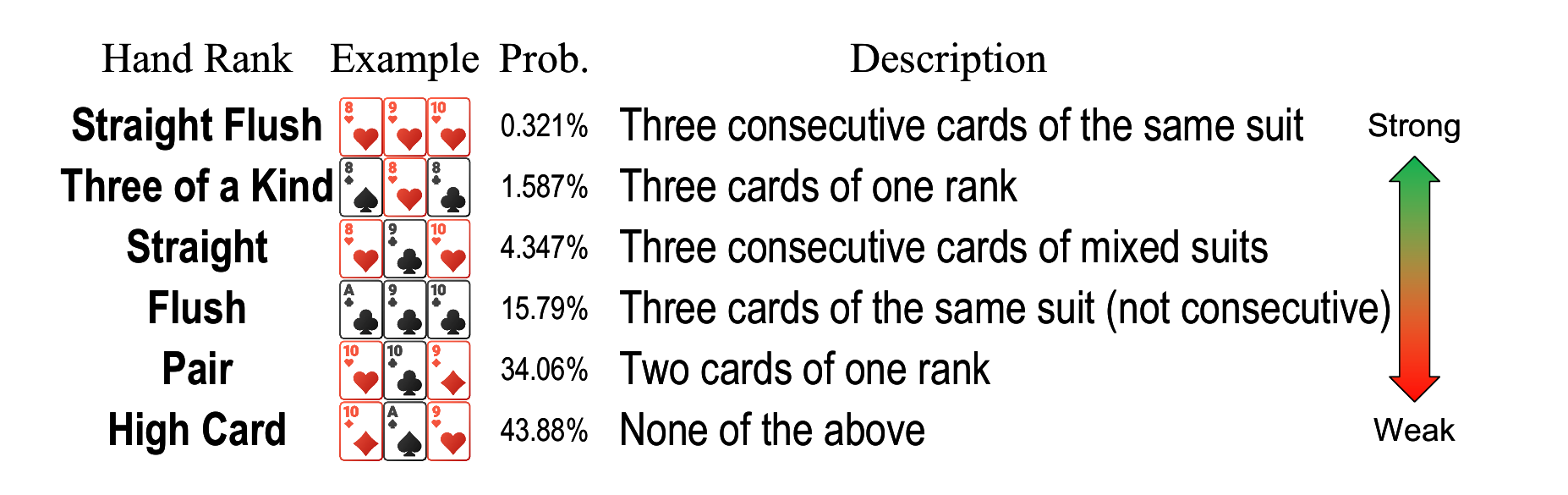}
\caption{The hand ranks of Numeral211 hold'em.}\label{fig:handranks}
\end{figure}

While our experiments use Numeral211 Hold’em, this variant shares core structural features with standard Texas Hold’em (e.g., HUNL)—including partial information, asymmetry, and phase-based decisions—with differences limited to hyperparameters (3 betting phases, 40-card deck, 3-card combinations). These adjustments do not alter hand abstraction logic, so our KrwEmd algorithm’s effectiveness here extends directly to real-world scenarios, requiring only hyperparameter matching (e.g., deck size) and accounting for reasonable computational overhead. Table \ref{tbl:krwi-kroi-numeral211} presents the number of signal observation infoset equivalence classes identified by different algorithms in Numeral211 Hold’em.

\begin{table}[ht]
  \centering
\resizebox{\columnwidth}{!}{%
\begin{tabular}{ccccccc}
\specialrule{1.2pt}{0pt}{0pt}
\hline
         & Phase 1 & \multicolumn{2}{c}{Phase 2} & \multicolumn{3}{c}{Phase 3}                   \\
         \cmidrule(lr){2-2} \cmidrule(lr){3-4} \cmidrule(lr){5-7} 
LI       & 100     & \multicolumn{2}{c}{2260} & \multicolumn{3}{c}{62020}       \\
         \cmidrule(lr){2-2} \cmidrule(lr){3-4} \cmidrule(lr){5-7} 
Recall (k)   & 0       & 0           & 1           & 0       & 1        & 2               \\ 
k-RWI     & 100     & 2234        & 2248     & 3957 & 51000 & 51070  \\
k-ROI     & 100     & 2250        & 2260     & 3957 & 51176 & 51228  \\
W/O (\%) & 100.0   & 99.29       & 99.47       & 100.0   & 99.67    & 99.69  \\ \hline   \bottomrule[1.2pt]
\end{tabular}%
}
\caption{The number of signal observation infoset equivalence classes identified by LI, k-RWI, and k-ROI in each phase of Numeral211 Hold'em, with W/O indicating the ratio of signal observation infoset equivalence classes identified by k-RWI to those identified by k-ROI.}
\label{tbl:krwi-kroi-numeral211}
\end{table}

\subsection{Experimental Setup}

Let $\alpha = (\alpha_1, \alpha_2)$ represent the signal observation abstraction to be assessed. Its strength is evaluated by measuring the \textbf{exploitability} , which is measured in terms of milli blinds (antes) per game (mb/g), of the approximate equilibrium derived using the CSMCCFR algorithm \citep{zinkevich2007regret, lanctot2009monte}, considering both symmetric and asymmetric abstraction scenarios.

In the \textbf{symmetric abstraction scenario}, commonly used in practical applications of abstraction, high-level AIs such as Libratus and DeepStack employ self-play to derive advanced strategies. We measure the exploitability of the approximate equilibrium yielded when both players in the game use signal observation abstraction. However, it may lead to the abstraction pathology \citep{waugh2009abstraction}. To avoid such problems, we illustrate the theoretical performance of the signal observation abstraction under evaluation through \textbf{asymmetric abstraction scenario}. In this scenario, the approximate equilibria of the signal-abstracted games $\tilde{\mathcal{G}}^{(\alpha_1, \vartheta_2)}$ and $\tilde{\mathcal{G}}^{(\vartheta_1, \alpha_2)}$ are computed, yielding strategy profiles $\sigma^{*,1}$ and $\sigma^{*,2}$, respectively. These two strategy profiles are then concatenated to form a strategy $\sigma' = (\sigma^{*,1}_1, \sigma^{*,2}_2)$, and the exploitability of $\sigma'$ is evaluated to assess the quality of the abstraction.

Regarding KrwEmd, we adopt a straightforward, intuitive distance matrix to quantify transitions between game outcomes:\[\mathbf{D} = \begin{bmatrix} 
0 & 1 & 2 \\ 
1 & 0 & 1 \\ 
2 & 1 & 0 
\end{bmatrix}\]In a two-player setting, its interpretation is deliberately simple: the first row (loss state) shows distances of 0 (loss→loss), 1 (loss→draw), and 2 (loss→win). This vanilla design uses distances to directly reflect intuitive "steps" between outcomes—e.g., a loss-to-win transition (two steps) carries twice the distance of a loss-to-draw shift (one step).

Our experiments were conducted on a device equipped with 754 GB of memory and an Intel(R) Xeon(R) Gold 6240R CPU (featuring a 2.4 GHz base frequency, a maximum turbo frequency of up to 4.00 GHz, and 96 threads).

\subsection{Results}

\begin{figure*}[htbp]
  \begin{minipage}[b]{0.49\textwidth}
    \centering
    \includegraphics[width=\linewidth]{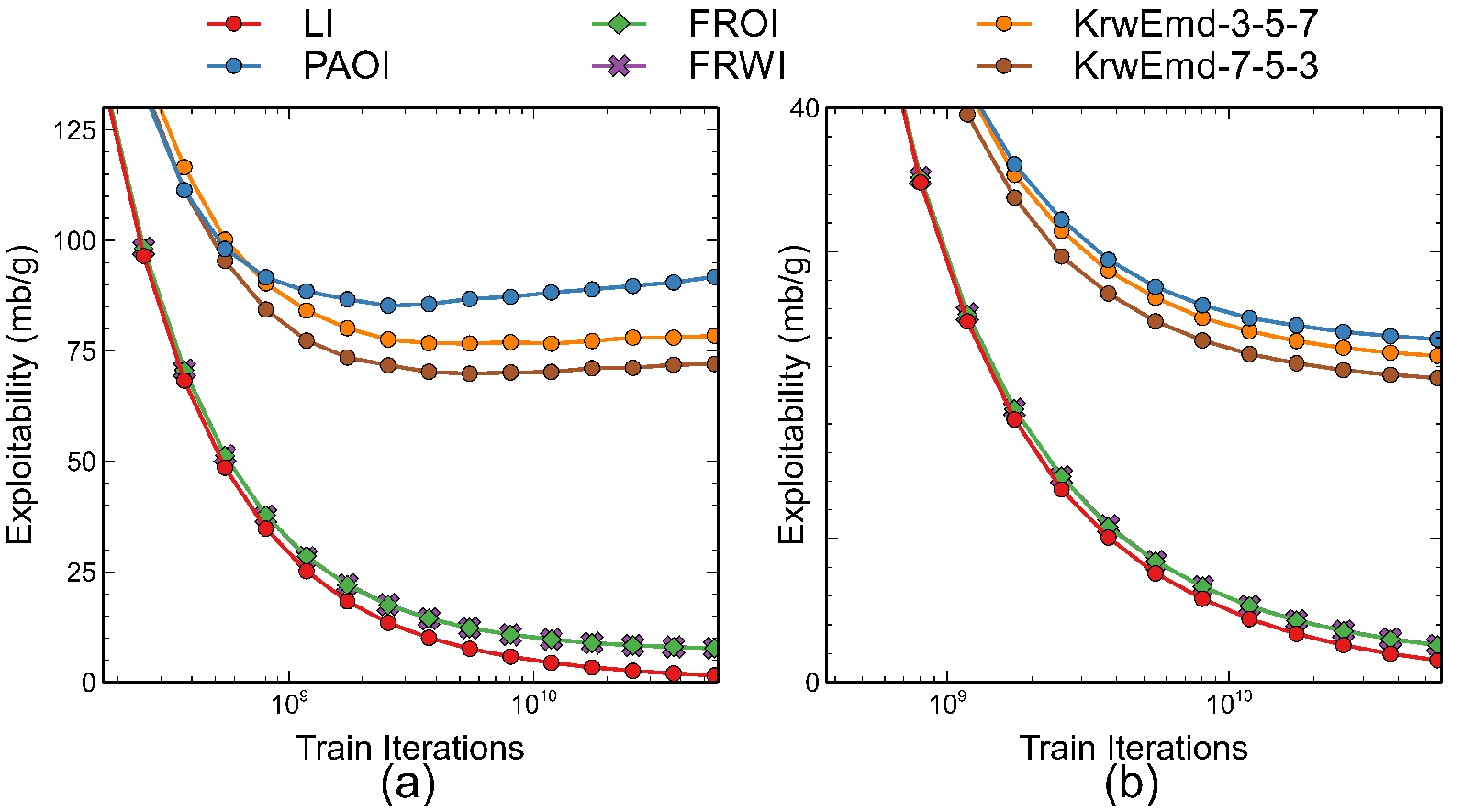}
    \caption{The isomorphism frameworks experiment was trained for $5.5\times 10^{10}$ iterations, with (a) representing the symmetric abstraction setting and (b) representing the asymmetric abstraction setting. Both instances of KrwEmd outperform PAOI, while the performance of 2-RWI and 2-ROI shows almost no difference in the Numeral211 environment.}
    \label{fig:isomorphism-exp}
  \end{minipage}
  \hspace{0.0\linewidth} 
  \begin{minipage}[b]{0.49\textwidth}
    \centering
    \includegraphics[width=\linewidth]{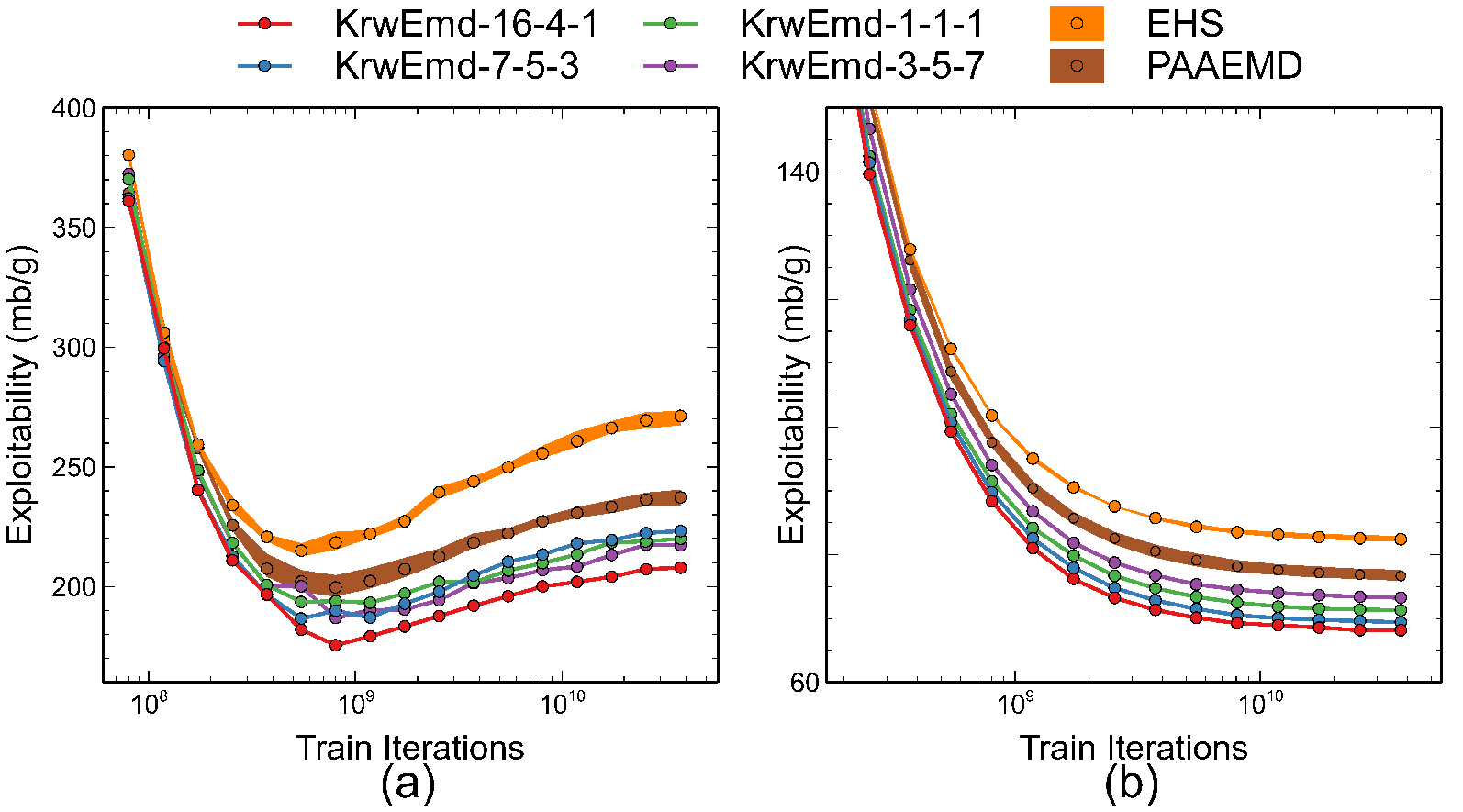}
    \caption{Performance comparison of KrwEmd versus other imperfect-recall signal observation abstraction algorithms considering only future information, trained for $3.7 \times 10^{10}$ iterations. All instances of KrwEmd outperform the benchmark, and comparisons between KrwEmd instances indicate that late-game information is more important than early-game information.}
    \label{fig:abstraction-exp}
  \end{minipage}
\end{figure*}

Firstly, we assess the performance of full-recall winrate isomorphism (FROI; which adopts $r-1$-RWI in phase-r) in comparison to other isomorphism frameworks—specifically FROI, PAOI and LI. Note that PAOI serves as the resolution bound for existing extreme implementations of imperfect-recall hand abstraction algorithms, which discard historical information entirely. Moreover, since previous works (KROI) could not control the number of abstracted signal observation infosets, they were unable to demonstrate whether incorporating historical information in signal observation abstraction outperformed abstraction with the same number of abstracted infosets. To address this, we included KrwEmd with the number of abstracted signal observation infosets set to match that of PAOI for a fair comparison in the isomorphism frameworks experiment. Note here that 0-RWI and 0-ROI exhibit identical capability in identifying signal infoset equivalence classes during Phase 1 (see Table \ref{tbl:krwi-kroi-numeral211}). Further, 0-RWI, 1-RWI, 0-ROI, and 1-ROI show quite small differences in recognizing these equivalence classes in Phase 2, as illustrated in Table \ref{tbl:krwi-kroi-numeral211}. Thus, we can directly allow clustering of KrwEmd abstraction use the signal observation infoset equivalence classes identified by PAOI in Phase 1 and 2, and only perform clustering in Phase 3. Here, we design four sets of hyperparameters \((w_0, w_1, w_2)\) for Equation \eqref{equa:distance-of-krwemd}, corresponding to different weight distributions for the importance of historical information: (1) sharply decreasing weights: \((16, 4, 1)\); (2) decreasing weights: \((7, 5, 3)\) (denoted as KrwEmd-7-5-3); (3) constant weights: \((1, 1, 1)\); and (4) increasing weights: \((3, 5, 7)\) (denoted as KrwEmd-3-5-7). For result visualization, Figure \ref{fig:isomorphism-exp}a presents the outcomes of symmetric abstraction, while Figure \ref{fig:isomorphism-exp}b shows those of asymmetric abstraction. To ensure the figures remain neat, only results from the best-performing (KrwEmd-7-5-3) and worst-performing (KrwEmd-3-5-7) hyperparameter sets are displayed. We observed that although the exploitability differed between the two experiments, the relative rankings of each group remained consistent (i.e., if A outperformed B in symmetric abstraction, it also did so in asymmetric abstraction). This consistent performance across experiments indicates the absence of abstraction pathology. As expected, overfitting was observed in the symmetric abstraction scenario, though it was only significant for PAOI. The performance difference between FRWI and FROI is small, which is related to the fact that the number of signal observation infoset equivalence classes identified by 2-RWI and 2-ROI in Numeral211 is similar (W/O generally exceeds 99\%). However, in HUNL—where the overlap rate (W/O) drops to around 90\%—we believe differences between the two exist, though such differences will not be significant. Most importantly, KrwEmd, outperforms PAOI—even with the worst hyperparameter (KrwEmd-3-5-7).


Next, we compared KrwEmd's performance with the currently applied hand abstraction algorithms, EHS and PAAEMD. It should be noted that PAOI is the resolution bound both for EHS and PAAEMD, meaning that the maximum number of abstracted signal observation infosets that EHS and PAAEMD can identify will not exceed the number of equivalence classes defined by PAOI. We set a compression rate that is 10 times lower than that of PAOI, while not performing abstraction for Preflop. The final number of abstracted signal observation infosets is set to 100, 225, 396 for Phase 1, 2 and 3. To exclude the influence of random events on performance, we generated 5 sets of abstractions for EHS and PAAEMD each. KrwEmd employs hyperparameters \((w_{3,0}, w_{3,1}, w_{3,2}; w_{2,0}, w_{2,1})\) for Phase 3 and Phase 2, respectively. These hyperparameter sets correspond to different weight distributions for the importance of historical information, as follows: (1) sharply decreasing  weights: \((16, 4, 1; 4, 1)\) (denoted as KrwEmd-16-4-1); (2) decreasing weights: \((7, 5, 3; 5, 3)\) (denoted as KrwEmd-7-5-3); (3) constant weights: \((1, 1, 1; 1, 1)\) (denoted as KrwEmd-1-1-1); and (4) increasing weights: \((3, 5, 7; 5, 7)\) (denoted as KrwEmd-3-5-7). Additionally, since PAAEMD uses approximate EMD calculations, its approximate distance is asymmetric, making it difficult for the algorithm to converge. We truncated after 1000 iterations on a single core, with an average cost of 1427.7s, while EHS and KrwEmd both achieved convergent clustering results, requiring an average of 12.3 and 96.7 iterations, with average time costs of 11.2s and 341.4s, respectively.


Figure \ref{fig:abstraction-exp}a presents the results of the symmetric abstraction setting, while Figure \ref{fig:abstraction-exp}b shows the results for the asymmetric abstraction setting. We observed that both symmetric and asymmetric abstractions maintained consistent performance, similar to the isomorphism frameworks experiment, without significant abstraction pathologies, despite noticeable overfitting in all abstraction algorithms under the symmetric setting. The experimental results indicate that KrwEmd significantly outperforms both EHS and PAAEMD across all parameter configurations. Furthermore, we validated that the importance of historical information decreases progressively from the late game to the early game. Notably, unlike the isomorphism frameworks experiment—where the optimal parameter used \textbf{decreasing weights}—the best-performing parameter in this case adopted \textbf{sharply decreasing weights}.

By providing a fair comparison, these two experiments validate that integrate historical information is indeed more effective than the approach of discarding historical information entirely in signal observation abstraction.

\section{Conclusion, Limitation, and Future Work} \label{sec:conclusion}

This research introduces the first imperfect-recall signal observation abstraction algorithm that considers historical information. This algorithm has the ability to adjust the scale of the abstracted signal observation infosets. Based on this, we fully verified that the imperfect-recall signal observation abstraction algorithms considering historical information is superior to that only considering future information. Imperfect-recall abstraction should be reexamined to introduce historical information and avoid excessive abstraction. Krwemd can help existing AIs achieve better performance.


There are two potential directions for future improvements. The first is to adopt distributed computing and approximation algorithms to reduce computational complexity. The second is to explore non-k-means algorithms and leverage machine learning techniques to incorporate historical information more effectively. Regardless of the approach, incorporating historical information in hand abstraction will help build more powerful poker game AI systems.

\bibliographystyle{plainnat}
\bibliography{references}

\begin{thebibliography}{29}
\providecommand{\natexlab}[1]{#1}
\providecommand{\url}[1]{\texttt{#1}}
\expandafter\ifx\csname urlstyle\endcsname\relax
  \providecommand{\doi}[1]{doi: #1}\else
  \providecommand{\doi}{doi: \begingroup \urlstyle{rm}\Url}\fi

\bibitem[Arthur and Vassilvitskii(2007)]{arthur2007k}
David Arthur and Sergei Vassilvitskii.
\newblock k-means++ the advantages of careful seeding.
\newblock In \emph{ACM-SIAM symposium on Discrete algorithms (SODA)}, pages
  1027--1035, 2007.

\bibitem[Billings et~al.(2003)Billings, Burch, Davidson, Holte, Schaeffer,
  Schauenberg, and Szafron]{billings2003approximating}
D~Billings, N~Burch, A~Davidson, R~Holte, J~Schaeffer, T~Schauenberg, and
  D~Szafron.
\newblock Approximating game-theoretic optimal strategies for full-scale poker.
\newblock In \emph{International Joint Conference on Artificial Intelligence
  (IJCAI)}, volume~3, pages 661--668, 2003.

\bibitem[Bonneel et~al.(2011)Bonneel, van~de Panne, Paris, and
  Heidrich]{BPPH11}
Nicolas Bonneel, Michiel van~de Panne, Sylvain Paris, and Wolfgang Heidrich.
\newblock {Displacement Interpolation Using Lagrangian Mass Transport}.
\newblock \emph{ACM Transactions on Graphics (SIGGRAPH ASIA 2011)}, 30\penalty0
  (6), 2011.

\bibitem[Brown and Sandholm(2018)]{brown2018superhuman}
Noam Brown and Tuomas Sandholm.
\newblock Superhuman ai for heads-up no-limit poker: Libratus beats top
  professionals.
\newblock \emph{Science}, 359\penalty0 (6374):\penalty0 418--424, 2018.

\bibitem[Brown and Sandholm(2019)]{brown2019superhuman}
Noam Brown and Tuomas Sandholm.
\newblock Superhuman ai for multiplayer poker.
\newblock \emph{Science}, 365\penalty0 (6456):\penalty0 885--890, 2019.

\bibitem[Brown et~al.(2019)Brown, Lerer, Gross, and Sandholm]{brown2019deep}
Noam Brown, Adam Lerer, Sam Gross, and Tuomas Sandholm.
\newblock Deep counterfactual regret minimization.
\newblock In \emph{International conference on machine learning (ICML)}, pages
  793--802. PMLR, 2019.

\bibitem[Brown et~al.(2020)Brown, Bakhtin, Lerer, and Gong]{brown2020combining}
Noam Brown, Anton Bakhtin, Adam Lerer, and Qucheng Gong.
\newblock Combining deep reinforcement learning and search for
  imperfect-information games.
\newblock \emph{Advances in Neural Information Processing Systems},
  33:\penalty0 17057--17069, 2020.

\bibitem[Fu et~al.(2025)Fu, Yin, Liu, Xu, and Huang]{fu2025beyond}
Yanchang Fu, Qiyue Yin, Shengda Liu, Pei Xu, and Kaiqi Huang.
\newblock Beyond outcome-based imperfect-recall: Higher-resolution abstractions
  for imperfect-information games.
\newblock \emph{arXiv preprint arXiv:2510.15094}, 2025.

\bibitem[Ganzfried and Sandholm(2014)]{ganzfried2014potential}
Sam Ganzfried and Tuomas Sandholm.
\newblock Potential-aware imperfect-recall abstraction with earth mover's
  distance in imperfect-information games.
\newblock In \emph{Proceedings of the AAAI Conference on Artificial
  Intelligence}, volume~28, 2014.

\bibitem[Gilpin and Sandholm(2008)]{gilpin2008expectation}
Andrew Gilpin and Thomas Sandholm.
\newblock Expectation-based versus potential-aware automated abstraction in
  imperfect information games: An experimental comparison using poker.
\newblock In \emph{National Conference on Artificial Intelligence (NCAI)},
  volume~3, pages 1454--1457, 2008.

\bibitem[Gilpin and Sandholm(2006)]{gilpin2006competitive}
Andrew Gilpin and Tuomas Sandholm.
\newblock A competitive texas hold'em poker player via automated abstraction
  and real-time equilibrium computation.
\newblock In \emph{National Conference on Artificial Intelligence (NCAI)},
  volume~21, page 1007. Menlo Park, CA; Cambridge, MA; London; AAAI Press; MIT
  Press; 1999, 2006.

\bibitem[Gilpin and Sandholm(2007{\natexlab{a}})]{gilpin2007better}
Andrew Gilpin and Tuomas Sandholm.
\newblock Better automated abstraction techniques for imperfect information
  games, with application to texas hold'em poker.
\newblock In \emph{International Joint Conference on Artificial Intelligence
  (IJCAI)}, pages 1--8, 2007{\natexlab{a}}.

\bibitem[Gilpin and Sandholm(2007{\natexlab{b}})]{gilpin2007lossless}
Andrew Gilpin and Tuomas Sandholm.
\newblock Lossless abstraction of imperfect information games.
\newblock \emph{Journal of the ACM (JACM)}, 54\penalty0 (5):\penalty0 25--es,
  2007{\natexlab{b}}.

\bibitem[Gilpin et~al.(2007)Gilpin, Sandholm, and
  S{\o}rensen]{gilpin2007potential}
Andrew Gilpin, Tuomas Sandholm, and Troels~Bjerre S{\o}rensen.
\newblock Potential-aware automated abstraction of sequential games, and
  holistic equilibrium analysis of texas hold'em poker.
\newblock In \emph{National Conference on Artificial Intelligence (NCAI)},
  volume~22, page~50. Menlo Park, CA; Cambridge, MA; London; AAAI Press; MIT
  Press; 1999, 2007.

\bibitem[Heinrich and Silver(2016)]{heinrich2016deep}
Johannes Heinrich and David Silver.
\newblock Deep reinforcement learning from self-play in imperfect-information
  games.
\newblock \emph{arXiv preprint arXiv:1603.01121}, 2016.

\bibitem[Kuhn(1950)]{kuhn1950simplified}
Harold~W Kuhn.
\newblock A simplified two-person poker.
\newblock \emph{Contributions to the Theory of Games}, 1:\penalty0 97--103,
  1950.

\bibitem[Lanctot et~al.(2009)Lanctot, Waugh, Zinkevich, and
  Bowling]{lanctot2009monte}
Marc Lanctot, Kevin Waugh, Martin Zinkevich, and Michael Bowling.
\newblock Monte carlo sampling for regret minimization in extensive games.
\newblock \emph{Advances in neural information processing systems}, 22, 2009.

\bibitem[Li et~al.(2020)Li, Hu, Zhang, Qi, and Song]{lidouble}
Hui Li, Kailiang Hu, Shaohua Zhang, Yuan Qi, and Le~Song.
\newblock Double neural counterfactual regret minimization.
\newblock In \emph{International Conference on Learning Representations
  (ICLR)}, 2020.

\bibitem[Liu et~al.(2023)Liu, Li, and Togelius]{liu2023model}
Weiming Liu, Bin Li, and Julian Togelius.
\newblock Model-free neural counterfactual regret minimization with bootstrap
  learning.
\newblock \emph{IEEE Transactions on Games}, 15\penalty0 (3):\penalty0
  315--325, 2023.

\bibitem[McAleer et~al.(2023)McAleer, Farina, Lanctot, and
  Sandholm]{DBLP:conf/iclr/McAleerFLS23}
Stephen~Marcus McAleer, Gabriele Farina, Marc Lanctot, and Tuomas Sandholm.
\newblock {ESCHER:} eschewing importance sampling in games by computing a
  history value function to estimate regret.
\newblock In \emph{International Conference on Learning Representations
  (ICLR)}, 2023.

\bibitem[Meng et~al.(2023)Meng, Ge, Tian, An, and Gao]{meng2023efficient}
Linjian Meng, Zhenxing Ge, Pinzhuo Tian, Bo~An, and Yang Gao.
\newblock An efficient deep reinforcement learning algorithm for solving
  imperfect information extensive-form games.
\newblock In \emph{Proceedings of the AAAI Conference on Artificial
  Intelligence}, volume~37, pages 5823--5831, 2023.

\bibitem[Morav{\v{c}}{\'\i}k et~al.(2017)Morav{\v{c}}{\'\i}k, Schmid, Burch,
  Lis{\`y}, Morrill, Bard, Davis, Waugh, Johanson, and
  Bowling]{moravvcik2017deepstack}
Matej Morav{\v{c}}{\'\i}k, Martin Schmid, Neil Burch, Viliam Lis{\`y}, Dustin
  Morrill, Nolan Bard, Trevor Davis, Kevin Waugh, Michael Johanson, and Michael
  Bowling.
\newblock Deepstack: Expert-level artificial intelligence in heads-up no-limit
  poker.
\newblock \emph{Science}, 356\penalty0 (6337):\penalty0 508--513, 2017.

\bibitem[Shi and Littman(2001)]{shi2001abstraction}
Jiefu Shi and Michael~L Littman.
\newblock Abstraction methods for game theoretic poker.
\newblock In \emph{Computers and Games: Second International Conference, CG
  2000 Hamamatsu, Japan, October 26--28, 2000 Revised Papers 2}, pages
  333--345. Springer, 2001.

\bibitem[Southey et~al.(2005)Southey, Bowling, Larson, Piccione, Burch,
  Billings, and Rayner]{southey2005bayes}
Finnegan Southey, Michael Bowling, Bryce Larson, Carmelo Piccione, Neil Burch,
  Darse Billings, and Chris Rayner.
\newblock Bayes' bluff: opponent modelling in poker.
\newblock In \emph{Proceedings of the Twenty-First Conference on Uncertainty in
  Artificial Intelligence}, pages 550--558, 2005.

\bibitem[Steinberger et~al.(2020)Steinberger, Lerer, and
  Brown]{steinberger2020dream}
Eric Steinberger, Adam Lerer, and Noam Brown.
\newblock Dream: Deep regret minimization with advantage baselines and
  model-free learning.
\newblock \emph{arXiv preprint arXiv:2006.10410}, 2020.

\bibitem[Waugh et~al.(2009{\natexlab{a}})Waugh, Schnizlein, Bowling, and
  Szafron]{waugh2009abstraction}
Kevin Waugh, David Schnizlein, Michael Bowling, and Duane Szafron.
\newblock Abstraction pathologies in extensive games.
\newblock In \emph{International Conference on Autonomous Agents and Multiagent
  Systems (AAMAS)}, volume~2, pages 781--788, 2009{\natexlab{a}}.

\bibitem[Waugh et~al.(2009{\natexlab{b}})Waugh, Zinkevich, Johanson, Kan,
  Schnizlein, and Bowling]{waugh2009practical}
Kevin Waugh, Martin Zinkevich, Michael Johanson, Morgan Kan, David Schnizlein,
  and Michael Bowling.
\newblock A practical use of imperfect recall.
\newblock In \emph{Symposium on Abstraction, Reformulation and Approximation
  (SARA)}, 01 2009{\natexlab{b}}.

\bibitem[Zhao et~al.(2022)Zhao, Yan, Li, Li, and Xing]{zhao2022alphaholdem}
Enmin Zhao, Renye Yan, Jinqiu Li, Kai Li, and Junliang Xing.
\newblock Alphaholdem: High-performance artificial intelligence for heads-up
  no-limit poker via end-to-end reinforcement learning.
\newblock In \emph{Proceedings of the AAAI Conference on Artificial
  Intelligence}, volume~36, pages 4689--4697, 2022.

\bibitem[Zinkevich et~al.(2007)Zinkevich, Johanson, Bowling, and
  Piccione]{zinkevich2007regret}
Martin Zinkevich, Michael Johanson, Michael Bowling, and Carmelo Piccione.
\newblock Regret minimization in games with incomplete information.
\newblock \emph{Advances in neural information processing systems}, 20, 2007.

\end{thebibliography}

\end{document}